\documentclass{article}
\usepackage{ijcai23}

\usepackage{times}
\usepackage{soul}
\usepackage{url}
\usepackage[hidelinks]{hyperref}
\usepackage[utf8]{inputenc}
\usepackage[small]{caption}
\usepackage{graphicx}
\usepackage{amsmath}
\usepackage{amsthm}
\usepackage{booktabs}
\usepackage{algorithm}
\usepackage[switch]{lineno}

\usepackage{natbib}

\usepackage{amssymb}
\usepackage{latexsym}
\usepackage{array}
\usepackage{longtable}

\usepackage{algpseudocode}

\title{On Orderings of Probability Vectors and Unsupervised Performance Estimation}

\author{
    Muhammad Maaz$^{1,2}$
    \and
    Rui Qiao$^1$
    \and
    Yiheng Zhou$^1$
    \And
    Renxian Zhang$^1$ \\
    \affiliations
    $^1$Amazon Science \\
    $^2$University of Toronto \\
    \emails 
    m.maaz@mail.utoronto.ca
}

\newtheorem{theorem}{Theorem}

\newtheorem{lemma}{Lemma}

\theoremstyle{definition}
\newtheorem{definition}{Definition}

\theoremstyle{remark}
\newtheorem*{remark}{Remark}

\newcommand{\R}{\mathbb{R}}
\newcommand{\E}{\mathbb{E}}
\newcommand{\I}{\mathbb{I}}
\newcommand{\Prob}{\mathbb{P}}
\newcommand{\norm}[2]{\|#1\|_#2}

\DeclareMathOperator{\range}{ran}

\DeclareMathOperator*{\argmin}{arg\,min}

\begin{document}

\maketitle

\begin{abstract}
	Unsupervised performance estimation, or evaluating how well models perform on unlabeled data is a difficult task. Recently, a method was proposed by \citet{garg2022} which performs much better than previous methods. Their method relies on having a score function, satisfying certain properties, to map probability vectors outputted by the classifier to the reals, but it is an open problem which score function is best. We explore this problem by first showing that their method fundamentally relies on the ordering induced by this score function. Thus, under monotone transformations of score functions, their method yields the same estimate. Next, we show that in the binary classification setting, nearly all common score functions - the $L^\infty$ norm; the $L^2$ norm; negative entropy; and the $L^2$, $L^1$, and Jensen-Shannon distances to the uniform vector - all induce the same ordering over probability vectors. However, this does not hold for higher dimensional settings. We conduct numerous experiments on well-known NLP data sets and rigorously explore the performance of different score functions. We conclude that the $L^\infty$ norm is the most appropriate.\footnote{Code available at: \url{https://github.com/mmaaz-git/acc-estim}.}
\end{abstract}

\section{Introduction}

When a machine learning model is trained, typically (in the supervised or semi-supervised cases), we have labels on the data and can easily calculate any number of popular metrics to evaluate the model's performance on some held-out test set, e.g., accuracy, true positive rate, precision, etc. Furthermore, performance on a labeled validation set can be used for model tuning. When labels are present, evaluating models is extremely straightforward.

However, in many real-world applications of machine learning, not all data is unlabeled, and in fact the amount of unlabeled data may vastly outnumber the amount of labeled data. For example, a chatbot may be trained on some number of queries that have been labeled by humans, but all possible queries cannot possibly be labeled. In many cases, human labeling of all data is too expensive, either costing too much time or too much money. However, for the deployer, knowing how well your model is performing on this unlabeled data is of prime importance, especially in the context of model performance degradation when faced with real-world data, which is often noisier and from a different distribution than the original data set.

More formally, we have a source domain on which the model is trained and a target domain on which it is deployed. This problem is referred to in the literature as unsupervised accuracy estimation, or more generally unsupervised performance estimation. Here, unsupervised refers to the fact that the target domain has no labels, and thus other information has to be used in order to evaluate the model's performance. Unsupervised performance estimation is related to the well-known problem of unsupervised domain adaptation, which seeks to appropriately deploy a model on to an unlabeled target domain, instead of explicitly estimating its performance. However, the two problems are quite related in spirit.

A recent contribution of \citet{garg2022} proposes a simple method for unsupervised performance estimation, called average thresholded confidence (ATC), which leverages the class probabilities outputted by a classifier. While their method outperforms previous baselines, several questions about the method are left open. In this paper, we provide several theoretical extensions of their work, which may be of independent interest, and expand our claims by testing their method on natural language processing (NLP) data sets.

\subsection{Related Literature}

Arguably, the theory of unsupervised performance estimation began with the seminal contribution of \citet{ben2010theory, ben2010impossibility}, which first showed that the error on a target domain can be bounded by a function of the error on the source domain and a measure of discrepancy between the source and target domains. In \citet{ben2010theory}, the $\mathcal{H}$-divergence is introduced as this measure of discrepancy. Later works explored other notions of distance and derived similar bounds.

Based on this notion, many methods of estimating accuracy have been proposed which essentially seek to align the source and target domains in a common feature space, and thus calculate the discrepancy between them. \citet{guillory2021predicting}, \citet{deng2021labels}, and \citet{deng2021does} use this notion, by finding a linear relationship between discrepancy and performance drop via a validation set. To quantify this discrepancy, \citet{deng2021labels} uses the Fr\'echet distance between feature vectors, while \citet{guillory2021predicting} uses differences in classifier probabilities.

\citet{jiang2022assessing} estimates error as the expected disagreement between two separately trained models, while \citet{chen2021mandoline} uses importance reweighting. It is again worth noting that both of these methods come from the domain adaptation literature. The notable work of \citet{platanios2016estimating} introduced a Bayesian approach which uses a graphical model to estimate the error.

These methods make some assumptions which may not necessarily hold. For example, \citet{guillory2021predicting} assumes a linear relationship between discrepancy and performance. While we may model a non-linear relationship as seen in \citet{deng2021labels}, if linearity does not hold, there still remains the issue that calculating discrepancy between data is fraught with complexities, especially in language data, where finding an appropriate feature space is difficult due to well-known issues of sparsity and loss of semantics with common embedding methods. It is perhaps for this reason that generative adversarial networks (GANs) are the most common method of domain adaptation for NLP models \citep{ramponi2020neural}. As well, the method of \citet{jiang2022assessing} doubles computational overhead by training two separate models.

Recently, \citet{garg2022} proposed the ATC method, which will be the focus of this paper. In their experiments, they outperformed the methods of \citet{guillory2021predicting}, \citet{deng2021labels}, \citet{jiang2022assessing}, and \citet{chen2021mandoline}.

\textbf{Our Contribution.} 
\citet{garg2022}'s method is parameterized by the choice of a certain \textit{score function}, whose properties we will discuss below. While they perform several experiments comparing two possible score functions, they still leave open some questions about their method. In this paper, we propose several appropriate score functions from the literature. We develop a theoretical basis for analyzing the estimations yielded by these different score functions. Based on this, we show that in the binary classification setting, (almost all) possible score functions yield the same estimate. Lastly, we perform multiple experiments in multiclass settings, showing with repeated experiments that the differences between these functions are small.

\section{Theory}

\subsection{Background}

Our setting is of a $k$-class classifier. To adopt the notation of \citet{garg2022}, we have an input domain $\mathcal{X} \subseteq \R^d$ and a label space $\mathcal{Y}$ containing $k$ distinct labels (e.g., for a binary classification problem, $\mathcal{Y}=\{0,1\}$). Denote $\mathcal{D}^s$ and $\mathcal{D}^t$ to be the source and target domains over $\mathcal{X} \times \mathcal{Y}$. We assume that we have labels in the source domain, but no labels in the target domain - i.e., it is an unsupervised domain adaptation problem. We have some performance metric $\Gamma^s$ and $\Gamma^t$ on the source and target domains, where they must be in the interval $[0,1]$. In \citet{garg2022}, they particularly take $\Gamma$ to be the classification error, but we will keep it more general for now. Let $\mathcal{F}$ be a set of hypotheses mapping $\mathcal{X} \mapsto \Delta^{k-1} \subset \R^k$, i.e., the input space to the $(k-1)$-dimensional probability simplex. Our goal is, given $f \in \mathcal{F}$, validation data from $D^s$, a source performance metric $\Gamma^s$, and unlabeled data from $D^t$, estimate $\Gamma^t$.

The method introduced by \citet{garg2022} depends on having a score function $s : \Delta^{k-1} \mapsto \R$ which takes the softmax output of $f \in \mathcal{F}$ and (roughly) gives a higher value if the classifier is more confident in its prediction. It should be minimized at the centroid of $\Delta^{k-1}$ (the uniform probability vector) and maximized on the vertices of $\Delta^{k-1}$, which corresponds to a single element of the vector being $1$ and all other elements being $0$. Two such functions studied by \citet{garg2022} are the maximum confidence $s(f(x)) = \max_{j \in \mathcal{Y}} f_j(x)$ and negative entropy $s(f(x)) = \sum_j f_j(x) \log(f_j(x))$. We study several such functions.

\citet{garg2022}'s method is as follows. First, learn a threshold $t$ such that:

\begin{equation} \label{source_metric}
    \Prob_{x \in D^s}[s(f(x))<t] = \Gamma^s
\end{equation}

Note that $\Prob[\bullet]$ is equivalent to $\E[\I[\bullet]]$. So we learn a threshold such that the proportion of points in $D^s$ that have a score less than $t$ is equal to our source performance metric. In practice, we do:

\begin{equation}
    t = \argmin_{t \in \range{s(f(x))}, \ x \in D^s} |\Gamma^s - \Prob_{x \in D^s}[s(f(x))<t]|
\end{equation}

where $\range$ is the range, i.e., all possible values of $s$ over $D^s$.

This means we calculate the threshold $t$ by minimizing the absolute difference between the two sides of Equation \ref{source_metric}, letting $t$ take on the values of the range of the score function outputted for the values in $D^s$.

Then we estimate $\Gamma^t$ as follows: 
\begin{equation}
    \Gamma^t = \Prob_{x \in D^t}[s(f(x))<t]
\end{equation}

\subsection{Orderings}

From the method identified above, we can see that the probability vectors are first mapped to a real number, and then ordered by the usual ordering of the reals. The threshold $t$ depends exactly on this ordering. Thus, the way the score function maps vectors to the reals determines the threshold. In particular, if for the same set of vectors we have two different score functions, then the estimate $\Gamma^t$ should be the same (albeit the actual threshold may differ). First, we introduce some terminology.

\begin{definition}[Order-isomorphism]
    Two score functions $s,\hat{s} : \Delta^{k-1} \mapsto \R$ are said to be \emph{order-isomorphic} if, for all $p,q \in \Delta^{k-1}$:
    \begin{itemize}
        \item $s(p)<s(q) \iff \hat{s}(p)<\hat{s}(q)$
        \item $s(p)=s(q) \iff \hat{s}(p)=\hat{s}(q)$
        \item $s(p)>s(q) \iff \hat{s}(p)>\hat{s}(q)$
    \end{itemize}
\end{definition}

The definition above, from order theory\footnote{Technically, an order-isomorphism is a bijective function defined between two posets that preserves and reflects orderings. This is slightly different from our setting, where we have a single set without an ordering ($\Delta^{k-1}$) and define two separate orderings on it.}. This captures the idea that if two different score functions do not change the ranking of a set of vectors, then they can be considered to be basically the same. Note that the third point in the definition is implied by the first two.

We now introduce our lemma.

\begin{lemma}[Ordering Lemma] \label{orderinglemma}
    If two score functions $s,\hat{s}$ are order-isomorphic then, for a fixed $f$, $D^s$, $D^t$, and $\Gamma^s$, we have that $\Gamma^t = \hat{\Gamma^t}$, where $\Gamma^t$ and $\hat{\Gamma^t}$ are the target performance estimates given by $s$ and $\hat{s}$ respectively.
\end{lemma}

\begin{proof}
    Denote $t$ and $\hat{t}$ the threshold obtained by $s$ and $\hat{s}$ respectively. Recall that $t$ and $\hat{t}$ take on the possible values of the range of their respective score functions, calculated for  data in $D^s$. Of course, they are possibly different. However, there exists a probability vector $p \in \range f(x)$ such that both $t=s(p)$ and $\hat{t}=\hat{s}(p)$. To see why this is, observe that the set $q \in \range f(x)$ such that $s(q)<s(p)$ is exactly the same as $\hat{s}(q)<\hat{s}(q)$ due to order-isomorphism. Thus, while $t$ and $\hat{t}$ are possibly different, they must come from the same vector $p$.
    
    Next, for $x \in D^t$, due to order-isomorphism, $s(f(x))<t=s(p)$ if and only if $\hat{s}(f(x))<\hat{t}=\hat{s}(p)$. Thus, the sets $\{s(f(x))<t \mid x \in D^t\}$ and $\{\hat{s}(f(x))<\hat{t} \mid x \in D^t\}$ are the same. Thus $\Gamma^t = \hat{\Gamma^t}$.
\end{proof}

The Ordering Lemma allows us to explore a variety of score functions, and show that they yield the same estimate simply by showing that they are order-isomorphic. In the next section, we do just that. We conclude with two lemmas that will play a role in future results.

\begin{lemma} \label{equiv}
    The property of being order-isomorphic is an equivalence relation on the set of score functions.
\end{lemma}
\begin{proof}
    A score function is clearly order-isomorphic with itself, and if $s$ is order-isomorphic with $\hat{s}$ then $\hat{s}$ is order-isomorphic with $s$. Lastly, if $s$ is order-isomorphic with $\hat{s}$ and $\hat{s}$ is order-isomorphic with $\hat{s}$, then $s$ is order-isomorphic with $\hat{s}$ by transitivity of the biconditional (see the definition of order-isomorphic). Hence, we have reflexivity, symmetry, and transitivity.
\end{proof}

In this paper, we use monotone in the order-theoretic sense, i.e., monotonically increasing, or order-preserving.

\begin{lemma} \label{monotransform}
    Consider two score functions $s$ and $\hat{s}$ such that $\hat{s} = g(s)$, where $g$ is strictly monotonic. Then $s$ and $\hat{s}$ are order-isomorphic.
\end{lemma}
\begin{proof}
    Follows from definition of monotonicity.
\end{proof}
\begin{remark}
    The above lemma seems quite obvious but we state it here to formally state that $\Gamma^t$ is invariant under monotone transformations of the score function. Namely, scaling, square root, etc., have no effect. This is important as there are some potential score functions, discussed in the section below, which are actually just monotone transformations of other score functions. This lemma allows us to collapse them as one case.
\end{remark}

\subsection{Exploring Score Functions}

For a practitioner of the ATC method, an important question that may arise is which score function to use. Indeed, the literature contains many examples of score functions that capture the rough intuition of being higher when more skewed and lower when more uniform. Below is a list of such score functions\footnote{Note that the use of score function here is not the same as a score function in decision theory.} seen in the literature, given a vector $p$ with components $(p_i)_1^k$:
\begin{itemize}
    \item $L^\infty$ norm, or maximum confidence: $\max_1^k p_i$
    \item Negative entropy: $\sum_1^k p_i \log(p_i)$
    \item $L^2$ norm: $\sqrt{\sum_1^k p_i^2}$. Of course, we can take the squared $L^2$, which is order-isomorphic due to strict monotonicty of square root (Lemma \ref{monotransform}). The squared $L^2$ is quite prevalent in many different fields: in economics it is called the Herfindahl-Hirschman index and measures market concentration \citep{rhoades1993herfindahl}, in ecology it is called the Simpson index and measures wildlife diversity \citep{simpson1949measurement}. The underlying principle is that it captures how uniform or skewed a probability vector is. The $L^2$ norm has a deep connection to the $\chi^2$ statistical test.
    \item $L^1$ distance to the uniform vector $\mathcal{U}_k = (1/k, \dots ,1/k)$: $\sum_1^k |p_i-1/k|$. Note that half the $L^1$ distance gives the total variation distance, and these two are clearly order-isomorphic (Lemma \ref{monotransform}).
    \item $L^2$ distance to $\mathcal{U}_k$: $\sqrt{\sum_1^k (p_i-1/k)^2}$
    \item Other distances or divergences to $\mathcal{U}_k$, e.g., Jensen-Shannon
    
\end{itemize}

Thus, the method of \citet{garg2022} has numerous possible implementations, depending on the chosen score function. It is not clear whether there is a best score function, or how to go about choosing one. Certainly, different score functions capture different information about the vector. The task for the practitioner then is seeing which score function works best for their particular case, through experimentation.

\subsection{The Case of Binary Classification}

While choosing the best score function seems to be a difficult problem that can only be solved experimentally, in the particular case of binary classification, or $\R^2$, we can get a simple characterization. We can show that a large class of natural score functions are order-isomorphic. This result shows that in the binary classification case, the choice becomes quite easy, and it follows to simply choose the most computationally cheap score function. In particular, we have the following theorem.

\begin{theorem}
	\label{R2isomorphic}
    In $\R^2$, on the probability simplex $\Delta^1$, the following score functions are order-isomorphic:
    \begin{enumerate}
        \item $L^\infty$ norm
        \item $L^2$ norm 
        \item $L^1$ distance to the uniform vector $(1/2,1/2)$
        \item $L^2$ distance to the uniform vector $(1/2,1/2)$
        \item negative entropy
        \item Jensen-Shannon distance to the uniform vector $(1/2,1/2)$
    \end{enumerate}
\end{theorem}

\begin{proof}
    Let $p,q \in \R^2$, particularly $p,q \in \Delta^1$, meaning that their $L^1$ norm is $1$. Parameterize them as $p=(a,1-a)$ and $q=(b,1-b)$. Without loss of generality, $a \geq 1-a$ and $b \geq 1-b$. Based on the $L^1$ norm constraint, $a,b \in [0.5,1]$. Observe that with this parameterization, $\norm{p}{\infty} = a$ and $\norm{q}{\infty}=b$.
    
    We will show that each of score functions (2) - (6) are order-isomorphic to (1), as numbered above. This will then imply that all of them are order-isomorphic to each other by Lemma \ref{equiv}.
    
    Our proof strategy is as follows: from our parameterization, we can consider the score function to only be a function of the first component (e.g., $a$ or $b$). If the score function is strictly monotone in this argument, then it follows that $a < b$ iff $s(a) < s(b)$. Another way of looking at it is that the score function is a strictly monotone transformation of $L^{\infty}$ and so by Lemma \ref{monotransform}, they are order-isomorphic.
    
    We now show this for each of the score functions listed above from (2) - (6).
    
    \begin{enumerate}
        \setcounter{enumi}{1}
        \item We work with the squared $L^2$ for ease of computation. We have $a^2 + (1-a)^2 = 2a^2 - 2a + 1$, which is strictly monotone on the interval $a \in [0.5,1]$. Thus, the result follows. To make this clear, observe that $\norm{p}{2} < \norm{q}{2} \iff 2a^2+1-2a < 2b^2+1-2b \iff a^2-a < b^2-b \iff a<b$, as the function $g(y)=y^2-y$ is strictly increasing for $y \in [0.5,1]$.
        \item $\norm{p-(1/2,1/2)}{1} = |a-0.5| + |(1-a)-0.5| = a-0.5 + 0.5-1+a = 2a-1$ which is strictly monotone in $a$ over $[0.5,1]$ hence the result follows.
        \item We work with squared $L^2$. We have $(a-0.5)^2 + (1-a-0.5)^2 = 2(a-0.5)^2 = 2a^2 - 2a + 0.5$, which is strictly monotone for $a \in [0.5,1]$.
        \item We have $a\log(a) + (1-a)\log(1-a)$. Taking the first derivative, we get $\log(a)-\log(1-a)$, which is $>0$ for $a \in [0.5,1]$ as $\log$ is strictly monotone. It then follows that $a \log(a) + (1-a) \log(1-a)$ is strictly monotone for $a \in [0.5,1]$.
        \item We have $\frac{1}{2}[a \log\frac{a}{0.5} + (1-a) \log\frac{1-a}{0.5} + 0.5 \log\frac{0.5}{a} + 0.5 \log\frac{0.5}{1-a}]$, which is strictly increasing on $a \in [0.5,1]$ (can be seen by taking the first derivative or graphing). 
    \end{enumerate}
    
    Thus, all of (2) - (6) are order-isomorphic with (1), and hence they are all order-isomorphic with each other.
    
\end{proof}

To summarize, the commonly used score functions --- $L^\infty$, $L^2$ (and its squared form, known variously as the Herfindahl-Hirschman index or the Simpson index), $L^1$ to uniform (and its scaled by $1/2$ form, the total variation distance), $L^2$ to uniform, negative entropy, and Jensen-Shannon to uniform --- all induce the same orderings over probability vectors in $\R^2$.

\subsection{The Case of Multi-Class Classification ($\mathbb{R}^{\geq 3}$)}

Unfortunately, Theorem \ref{R2isomorphic} does not hold for $\R^3$. For a simple counterexample showing that $L^2$ and $L^\infty$ are not order-isomorphic, consider $p=(0.5,0.2,0.3)$ and $q=(0.5,0.5,0)$. Then, $\norm{p}{2}^2=0.38$, $\norm{q}{2}^2=0.5$, and $\norm{p}{\infty} = \norm{q}{\infty} = 0.5$. So $\norm{p}{2}<\norm{q}{2}$ but $\norm{p}{\infty} \nless \norm{q}{\infty}$. It is possible to construct counterexamples for the other score functions as well. As well, as $\mathbb{R}^3$ is a subspace of $\mathbb{R}^{>3}$, it does not hold for $\mathbb{R}^{\geq 3}$.

The order-isomorphism that still holds in $\mathbb{R}^{\geq 3}$ is between the $L^2$ norm and $L^2$ distance to uniform, shown below.

\begin{theorem}
	\label{l2equiv}
	In $\mathbb{R}^k$, on the probability simplex $\Delta^{k-1}$, the score functions $\norm{\bullet}{2}$ ($L^2$ norm) and $\norm{\bullet-1/k}{2}$ ($L^2$ distance to uniform) are order-isomorphic.
\end{theorem}

\begin{proof}
	Let $p \in \Delta^{k-1}$ be written as $(p_i)_1^k$. We have the constraint that $\sum_1^k p_i = 1$. To enforce this constraint, we can substitute $p_k=1-\sum_1^{k-1} p_i$. For ease of computation, we work with the squared forms of the score functions, which we denote by $s$ and $\hat{s}$ respectively. So $s(p) = \sum_1^{k-1} p_i^2 + (1-\sum_1^{k-1} p_i)^2$, and $\hat{s}(p) = \sum_1^{k-1}(p_i-\frac{1}{k})^2 + (1-\frac{1}{k}-\sum_1^{k-1} p_i)^2$.
	
	Instead of expanding these expressions, we can analyze their gradients. Now, $s$ and $\hat{s}$ are functions of $k-1$ independent variables (because of our substitution for $p_k$). For some variable $p_j$, we have $\frac{\partial s}{\partial p_j} = 2p_j-2(1-\sum_1^{k-1}p_i)$, and $\frac{\partial \hat{s}}{\partial p_j} = 2(p_j-\frac{1}{k})-2(1-\frac{1}{k}-\sum_1^{k-1}p_j)$. By expanding the $1/k$ term, we can cancel them. Thus, $\forall p_i$, $\frac{\partial s}{\partial p_i} = \frac{\partial \hat{s}}{\partial p_i}$. 
	
	Hence, $\nabla s = \nabla \hat{s}$, and by the mean value theorem, $s$ and $\hat{s}$ differ only by a constant, and thus $\hat{s}$ is a monotone transformation of $s$, which by Lemma \ref{monotransform} completes the proof.
\end{proof}

\begin{remark}
	In fact, Theorem \ref{l2equiv} holds for the $L^2$ distance to a vector whose components are all the same, as this will allow the gradients to be equal. As we are on $\Delta^{k-1}$, the only such vector is the uniform probability vector. In general, $L^2$ norm and $L^2$ distance to a fixed vector are not order-isomorphic. For example, fix $r=(0.3,0.7)$. Take $p=(0.4,0.6)$ and $q=(0.5,0.5)$. Then observe $\norm{p}{2}^2=0.52 > \norm{q}{2}^2=0.5$, but $\norm{p-r}{2}^2=0.02 < \norm{q-r}{2}^2=0.08$.
\end{remark}

However, because in general the other score functions from Theorem \ref{R2isomorphic} do not induce the same orderings, we cannot guarantee that they will yield the same estimate. Hence, in the multi-class setting, we once again have to ponder which score function will give us the best estimate. As we cannot know this theoretically, in the next section, we run experiments on three well-known NLP data sets with different score functions.

\section{Experiments}

\subsection{Methods}

We ran experiments on three NLP data sets for multi-class classification: Emotion \citep{saravia2018carer}, a data set of tweets labeled with one of 6 emotions, TweetEval \citep{barbieri2018semeval}, a data set of tweets labeled with one of 20 emojis, and Banking77 \citep{casa2020}, a data set of online banking queries labeled with one of 77 intents. The data sets came with train, validation, and testing splits; if not, a validation set of 20\% of the training set was created and set aside.

The basic schema of the experiment is that we finetuned DistilBERT-base-uncased \citep{sanh2019distilbert} on the training set. The hyperparameters used are learning rate: 2e-5; batch size: 256; epochs: 5; weight decay: 0.01. We then used the validation set to estimate performance on the testing set, and then actually tested the model on the testing set. We calculated and stored the absolute difference between the estimated accuracy and the true accuracy. We present our methodology in the algorithm box.

\begin{algorithm}
\caption{Experiment pseudocode}
\begin{algorithmic}
\Require Dataset with $k$ classes labeled $\{0, 1, ..., k-1\}$, split into training $\mathcal{T}$, testing $\mathcal{S}$, and validation sets $\mathcal{V}$; pre-trained DistilBERT-base-uncased model

\State Results $\leftarrow$ empty list

\For{$d \in \{2,3,...k\}$}
\State $T \leftarrow$ subset of $\mathcal{T}$ with labels $< d$
\State $S \leftarrow$ subset of $\mathcal{S}$ with labels $< d$
\State $V \leftarrow$ subset of $\mathcal{V}$ with labels $< d$

\State Finetune DistilBERT-base-uncased on $T$

\State $a \leftarrow$ accuracy of model on $S$

\For{$m$ in \{ATC-Max, ATC-NE, ATC-L2n, ATC-L1, ATC-L2, ATC-JS, DoC\}}

\For{$n=1$ to $1000$}
\State $V^* \leftarrow$ bootstrap $V$ with $|V|$ samples

\State $\tilde{a} \leftarrow$ estimated accuracy on $S$ using $V^*$ using method $m$

\State $\epsilon \leftarrow |a - \tilde{a}|$

\State Append list $[i,m,n,\epsilon]$ to Results

\EndFor

\EndFor
\EndFor

\State \Return Results

\end{algorithmic}
\end{algorithm}

We tested 6 different score functions for \citet{garg2022}'s method, which is referred to in the results as ATC (average thresholded confidence) per their paper. The score functions we tested are: $L^\infty$ (ATC-Max), negative entropy (ATC-NE), $L^2$ norm (ATC-L2n), $L^1$ distance to uniform (ATC-L1), $L^2$ distance to uniform (ATC-L2), and Jensen-Shannon distance to uniform (ATC-JS). We also implemented an older baseline for accuracy estimation, the difference-in-confidence (DoC) method \citep{guillory2021predicting}, as a comparison. Briefly, this method uses the difference in average of max confidence (hence, DoC) as a distance between distributions. By fitting a linear regression using validation data, this method estimates accuracy on the test set by the DoC between the test set and a held-out set.

We also wanted to explore the performance of different score functions for varying numbers of dimensions. The data sets we use have 6, 20, and 77 dimensions respectively. To simulate other dimensions, we simply just took the first $\hat{k}$, $2 \leq \hat{k} \leq k$ classes, where $k$ is the total number of classes: e.g., the Emotion data set has labels $\{0,1,...,5\}$, we take $\{0,1\}, \{0,1,2\}, \{0,1,2,3\}$, etc, to make, respectively, a 2-class, 3-class, a 4-class problem, etc.

Therefore, for each dimension, we tried 7 total methods. For each dimension-method pair, we did 1000 runs of the experiment, by taking a subset of the validation set via bootstrap sampling for each run. This makes our results more statistically robust and allows us to quantify a confidence interval.

\subsection{Results}

For sake of space, we only present the full table of results for the Emotion experiment, which we see in Table \ref{resemotion}. The full table of results for TweetEval and Banking77 are in the appendix. The bolded value corresponds to the lowest mean absolute error for that dimension. First of all, ATC outperforms the DoC baseline no matter the chosen score function. Second of all, as expected per Theorem \ref{R2isomorphic}, for 2 dimensions, all the ATC score functions yielded exactly the same result for every run. Looking at Table \ref{resemotion}, the best score function is still unclear. It seems that Jensen-Shannn performs best for dimensions 3, 4, and 6, but is beat for dimension 4. It seems the best as it wins in most dimensions.

\renewcommand{\arraystretch}{1.5}
\begin{table*}[!htbp] \centering 
	\caption{Mean [95\% confidence interval] absolute error of accuracy prediction for Emotion data set. Bolded cells are the lowest mean absolute error for that dimension.} 
	\label{resemotion} 
	\begin{tabular}{c*{7}{>{\centering\arraybackslash}p{1.5cm}}} 
		\toprule
		\# Dimensions & ATC-JS & ATC-L1 & ATC-L2 & ATC-L2n & ATC-Max & ATC-NE & DoC \\ 
		\midrule
		2 & \textbf{0.49 [0.16,1.02]} & \textbf{0.49 [0.16,1.02]} & \textbf{0.49 [0.16,1.02]} & \textbf{0.49 [0.16,1.02]} & \textbf{0.49 [0.16,1.02]} & \textbf{0.49 [0.16,1.02]} & 1.98 [0.06,5.79] \\ 
		3 & \textbf{1.28 [0.00,3.07]} & 1.37 [0.07,2.93] & 1.37 [0.07,2.93] & 1.38 [0.07,2.93] & 1.37 [0.07,2.93] & 1.29 [0.07,3.00] & 3.09 [0.11,8.14] \\ 
		4 & \textbf{1.31 [0.06,2.69]} & 1.33 [0.12,2.75] & 1.33 [0.06,2.75] & 1.33 [0.06,2.75] & 1.33 [0.12,2.75] & 1.32 [0.12,2.69] & 3.11 [0.14,8.20] \\ 
		5 & 0.84 [0.05,2.48] & \textbf{0.66 [0.00,1.71]} & \textbf{0.66 [0.05,1.76]} & \textbf{0.66 [0.05,1.76]} & \textbf{0.66 [0.00,1.71]} & 0.87 [0.05,2.02] & 2.73 [0.14,7.82] \\ 
		6 & \textbf{1.22 [0.25,2.65]} & 1.33 [0.30,2.85] & 1.33 [0.30,2.85] & 1.33 [0.30,2.85] & 1.33 [0.30,2.85] & 1.26 [0.20,2.80] & 2.97 [0.11,7.74] \\ 
		\bottomrule
	\end{tabular} 
\end{table*} 

However, this is only for the Emotion data set. Indeed, in experiments on TweetEval and Banking77, a rather different picture emerges. To look at this more rigorously, we calculated, for each of the three data sets, the proportion of dimensions that each method yields the lowest mean absolute error. In this calculation, we exclude the 2-class setting, per Theorem \ref{R2isomorphic}. We consider a tie by awarding all tied functions a win; thus the counts may not add up to the total number of dimensions being used as ``rounds" in this comparison. 

\renewcommand{\arraystretch}{1}
\begin{table}[!htbp] \centering 
	\caption{Number of dimensions for which each method achieves the lowest mean absolute error. Total is the number of dimensions for which comparisons were made (excludes 2-class setting). Existence of ties means that counts do not necessarily add to total.}
	\label{bestmethod} 
	\begin{tabular}{cccc} 
		\toprule
		Method & Emotion & TweetEval & Banking77 \\
		\midrule
		ATC-JS & 3 & 1 & 2 \\
		ATC-L1 & 1 & 5 & 12 \\
		ATC-L2 & 0 & 4 & 9 \\
		ATC-L2n & 0 & 4 & 2 \\
		ATC-Max & 0 & 7 & 39 \\
		ATC-NE & 0 & 1 & 4 \\
		DoC & 0 & 0 & 8 \\
		\bottomrule
		\textit{Total} & 4 & 18 & 75 \\
		\bottomrule
	\end{tabular} 
\end{table} 

Some interesting observations can be noted here. First, ATC-JS performs poorly for TweetEval and Banking77, contradicting our initial excitement with Emotion. Secondly, ATC-Max wins a plurality of dimensions for TweetEval, and wins an even larger plurality for Banking77. Interestingly, Banking77 is the only data set for which DoC bests the ATC methods at least once. This may be because it is a more complicated data set. In fact, the mean absolute error on Banking77 is much higher than the other data sets for the same method and dimension. 

We also see, as expected by Theorem \ref{l2equiv}, that ATC-L2 and ATC-L2n yield equal (or equal up to several decimal places, reflecting computational precision) estimates. For TweetEval, ATC-L2 is always tied with ATC-L2n, while for Banking77 they are within $10^{-6}$ of each other.

Now, in our discussion so far we have only been comparing the actual means of the absolute error. However, as we took bootstrap samples, we have a bootstrap distribution. We can therefore use the entire distribution of errors to make comparisons. Figure \ref{tweeteval_fig} shows violin plots of the distribution of errors for each method and dimension for the TweetEval data set. Firstly, we can see that the DoC distribution is much wider than the ATC distributions. They also show considerable overlap between the distributions, and so it is difficult to say whether one is really lower or higher than the other. We are therefore interested in using this distribution to test if the differences between methods are statistically significant.

\begin{figure*}
	\label{tweeteval_fig}
	\caption{Distribution of absolute errors for TweetEval, represented as violin plot with overlayed box plot to show quartiles. Headings refer to number of dimensions.}
	\includegraphics[width=\textwidth]{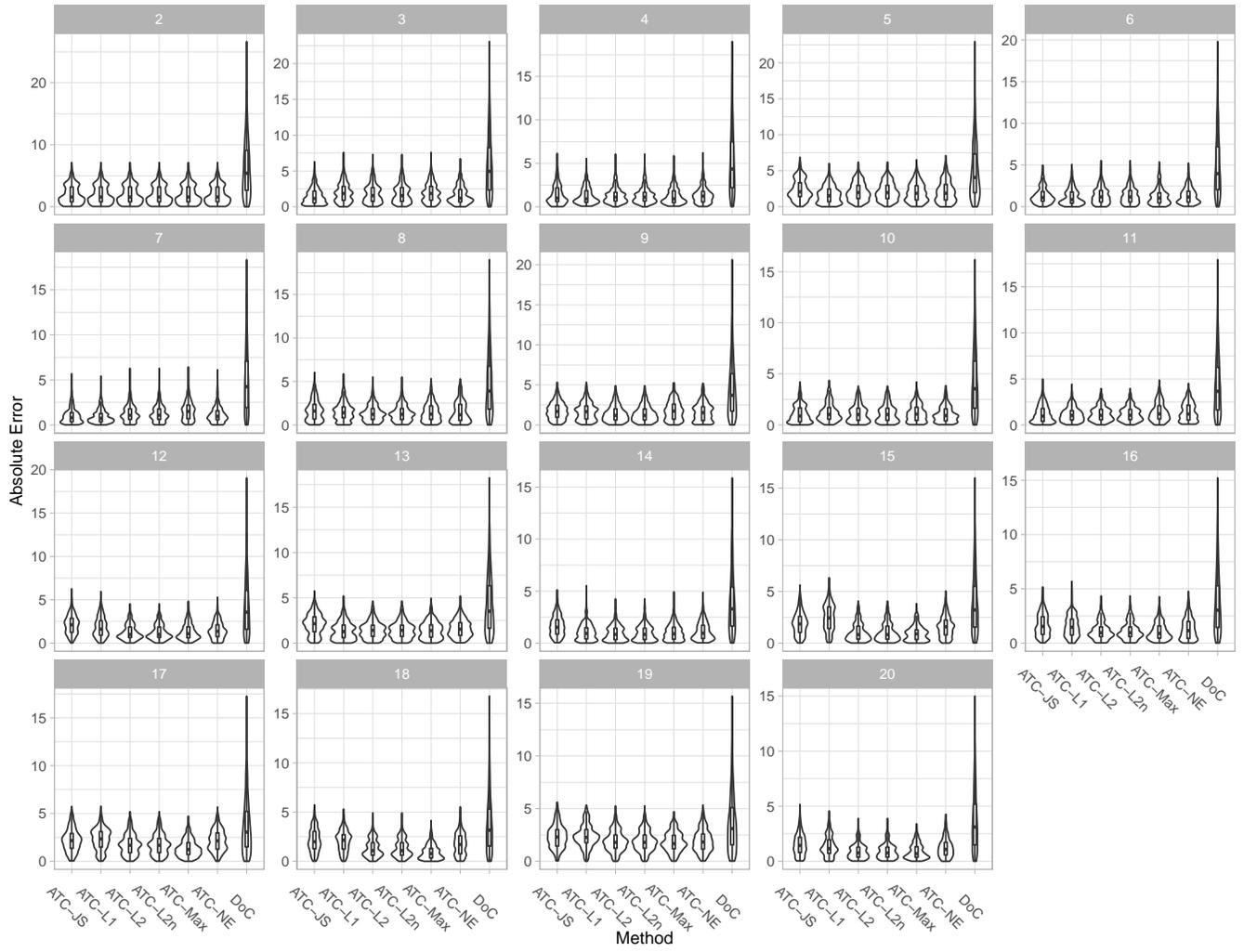}
\end{figure*}

Hence, we performed ANOVA tests on the experimental results for our three data sets. We performed a two-way ANOVA on the mean absolute error, using the dimension and method as independent variables. Our ANOVA was appropriately powered at a significance level of 0.05, power 0.99, and minimum detectable effect size of 0.1 (what \citet{cohen2013statistical} calls a ``small" effect size for the $f^2$ statistic). For all three data sets, the two-way ANOVA detected a significant difference between methods. We then performed post-hoc analysis with Tukey's Honestly Significant Difference (HSD) method, which tests pairwise differences. For Emotion, the only significant differences were seen between DoC and the other ATC methods. For TweetEval and Banking77, significant differences were observed for all pairwise comparisons except for ATC-L2 and ATC-L2n, as expected due to Theorem \ref{l2equiv}. It may be that significant differences were yielded for TweetEval and Banking77 as the experimental results data set is much larger than for Emotion (e.g., for Banking77, we have 1000 runs $\times$ 76 possible dimensions $\times$ 7 methods $=$ 532,000 data points). Looking at the comparisons between score functions for ATC, even when there is a significant difference, the differences are on the magnitude of $10^{-1}$, suggesting that different score functions, while they may yield different estimates, yield estimates within one decimal point of each other.

\section{Discussion and Conclusion}

In this paper we explored the average thresholded confidence (ATC) method of \citet{garg2022}, specifically the question of how to choose the score function, an important hyperparameter in their method. We presented some theoretical results about how different score functions can give the same estimate. Our results on orderings of probability vectors, namely that several score functions induce the same ordering, may also be of independent interest. Because a clean theoretical result does not exist for {$\mathbb{R}^{\geq 3}$, we conducted several large-scale experiments on three NLP data sets to analyze which score function best predicts the true accuracy on a test set. Our ANOVA results show that while there are significant differences between the score functions, they are on the magnitude of $10^{-1}$, which suggests that the choice of score function may not be important. Therefore, we would suggest to use the computationally fastest score function, which, in our benchmarks, is the $L^\infty$ norm. In our tests, this was implemented in Python by the numpy library. We compared this against all the other score functions, variously implemented in numpy and scipy. The negative entropy and Jensen-Shannon were by far the slowest, likely owing to the difficulty in calculating logarithms. However, compared to the computational cost of inference, these differences may be negligible.

Therefore, for a practitioner of \citet{garg2022}'s method, we conclude that the $L^\infty$ is the best score function to use. In the binary case, the choice definitely does not matter; and in the multi-class case, it may, but not so much, and the $L^\infty$ is fastest to compute. 

It is important to note that our comparisons between methods were predicated on taking bootstrap samples of the validation set, and calculating a mean, which the original paper of \citet{garg2022} did not do. We therefore suggest that this is a simple modification to their method which decreases random error.

There are still several interesting avenues of study. First of all, it is not clear why the score function in \citet{garg2022}'s method has to measure the ``confidence" of the classifier. A justification is not provided in \citet{garg2022}, and our results simply show the importance of the induced ordering. Elucidating if the score function needs to necessarily fulfill this property, and why, would be a gainful line of research. The ATC method outperforms several previous baselines in \citet{garg2022}, and we reinforce that in our paper by showing it outperforms DoC - however, it is still unclear why it seems to work so well. As well, the method of \citet{garg2022} is quite new, and thus has not necessarily been corroborated and tested extensively. In this paper we tested its performance on NLP data sets. It will be interesting to see if the superiority of ATC continues.

\section*{Acknowledgements}

Muhammad Maaz was an intern at Amazon at the time this work was done. Muhammad Maaz would also like to acknowledge support through the Vanier Canadian Graduate Scholarship, administered by the Natural Sciences and Engineering Research Council of Canada.

\newpage
\bibliographystyle{unsrtnat}
\bibliography{ref}

\onecolumn
\appendix
\section{Appendix}

\subsection{Full results on TweetEval}

Below are the full results on the TweetEval dataset.

% Table created by stargazer v.5.2.2 by Marek Hlavac, Harvard University. E-mail: hlavac at fas.harvard.edu
% Date and time: Sun, Apr 23, 2023 - 3:45:51 AM
\begin{table}[!htbp] 
\small
\centering 
  \caption{Mean [95\% confidence interval] absolute error of accuracy prediction for TweetEval data set.} 
  \label{restweeteval} 
\begin{tabular}{c*{7}{>{\centering\arraybackslash}p{1.5cm}}} 
\toprule
\# Dimensions & ATC-JS & ATC-L1 & ATC-L2 & ATC-L2n & ATC-Max & ATC-NE & DoC \\ 
\midrule
2 & 2.00 [0.08, 5.24] & 2.00 [0.08, 5.24] & 2.00 [0.08, 5.24] & 2.00 [0.08, 5.24] & 2.00 [0.08, 5.24] & 2.00 [0.08, 5.24] & 6.32 [0.23,17.69] \\ 
3 & 1.46 [0.10, 4.14] & 1.94 [0.09, 4.67] & 1.85 [0.09, 4.95] & 1.85 [0.09, 4.95] & 1.94 [0.09, 4.67] & 1.62 [0.04, 4.44] & 5.80 [0.28,16.10] \\ 
4 & 1.38 [0.03, 4.22] & 1.27 [0.11, 3.55] & 1.29 [0.10, 3.52] & 1.29 [0.10, 3.52] & 1.24 [0.10, 3.42] & 1.35 [0.07, 3.68] & 5.22 [0.26,14.62] \\ 
5 & 2.36 [0.22, 5.20] & 1.69 [0.05, 4.01] & 2.11 [0.22, 4.67] & 2.11 [0.22, 4.67] & 1.98 [0.06, 4.67] & 2.10 [0.07, 5.34] & 4.92 [0.20,13.85] \\ 
6 & 1.33 [0.12, 3.71] & 1.17 [0.02, 3.50] & 1.37 [0.03, 3.77] & 1.37 [0.03, 3.77] & 1.23 [0.06, 3.86] & 1.34 [0.11, 3.77] & 4.98 [0.34,13.74] \\ 
7 & 1.00 [0.03, 3.06] & 0.95 [0.03, 2.84] & 1.27 [0.05, 3.40] & 1.27 [0.05, 3.40] & 1.56 [0.08, 3.92] & 1.15 [0.05, 3.52] & 4.86 [0.20,14.07] \\ 
8 & 1.65 [0.09, 4.17] & 1.53 [0.07, 3.85] & 1.34 [0.04, 3.61] & 1.34 [0.04, 3.61] & 1.46 [0.06, 3.87] & 1.48 [0.03, 4.28] & 4.67 [0.20,12.95] \\ 
9 & 1.78 [0.09, 4.12] & 1.68 [0.15, 3.94] & 1.40 [0.03, 3.82] & 1.40 [0.03, 3.82] & 1.80 [0.12, 4.18] & 1.61 [0.08, 4.36] & 4.50 [0.23,13.12] \\ 
10 & 1.07 [0.03, 3.07] & 1.23 [0.10, 3.35] & 1.12 [0.04, 3.13] & 1.12 [0.04, 3.13] & 1.16 [0.07, 2.99] & 1.05 [0.06, 2.97] & 4.24 [0.20,11.53] \\ 
11 & 1.18 [0.05, 3.22] & 1.13 [0.03, 2.84] & 1.20 [0.11, 3.01] & 1.20 [0.11, 3.01] & 1.39 [0.02, 3.52] & 1.43 [0.09, 3.44] & 4.25 [0.13,11.72] \\ 
12 & 2.14 [0.41, 4.35] & 1.83 [0.15, 4.25] & 1.26 [0.08, 3.28] & 1.26 [0.08, 3.28] & 1.31 [0.07, 3.47] & 1.47 [0.04, 3.36] & 4.17 [0.15,11.99] \\ 
13 & 2.07 [0.08, 4.29] & 1.37 [0.05, 3.55] & 1.44 [0.07, 3.61] & 1.44 [0.07, 3.61] & 1.44 [0.05, 3.50] & 1.62 [0.15, 3.58] & 4.24 [0.19,11.63] \\ 
14 & 1.65 [0.13, 3.69] & 1.05 [0.06, 2.93] & 0.96 [0.04, 2.75] & 0.96 [0.04, 2.75] & 1.06 [0.07, 2.82] & 1.15 [0.07, 2.92] & 3.90 [0.20,11.58] \\ 
15 & 1.87 [0.15, 4.16] & 2.44 [0.12, 4.84] & 1.05 [0.04, 3.18] & 1.05 [0.04, 3.18] & 0.93 [0.01, 2.67] & 1.56 [0.05, 3.69] & 3.89 [0.20,10.79] \\ 
16 & 1.69 [0.20, 3.90] & 1.57 [0.05, 3.61] & 1.14 [0.05, 2.96] & 1.14 [0.05, 2.96] & 1.08 [0.08, 2.85] & 1.32 [0.05, 3.51] & 3.68 [0.15,10.58] \\ 
17 & 2.19 [0.12, 4.70] & 2.32 [0.28, 4.49] & 1.74 [0.12, 4.22] & 1.74 [0.12, 4.22] & 1.37 [0.14, 3.40] & 2.15 [0.24, 4.52] & 3.68 [0.14,10.27] \\ 
18 & 2.14 [0.13, 4.36] & 2.06 [0.15, 4.04] & 1.35 [0.06, 3.21] & 1.35 [0.06, 3.21] & 0.93 [0.04, 2.74] & 1.74 [0.08, 4.01] & 3.75 [0.12,10.78] \\ 
19 & 2.34 [0.36, 4.64] & 2.36 [0.31, 4.53] & 1.90 [0.19, 4.07] & 1.90 [0.19, 4.07] & 1.84 [0.22, 3.67] & 1.94 [0.22, 4.28] & 3.64 [0.17,10.41] \\ 
20 & 1.52 [0.08, 3.56] & 1.38 [0.02, 3.42] & 0.92 [0.03, 2.54] & 0.92 [0.03, 2.54] & 0.91 [0.06, 2.49] & 1.24 [0.02, 3.38] & 3.67 [0.12,10.11] \\ 
\bottomrule 
\end{tabular} 
\end{table}

\subsection{Full results on Banking77}

Below are the full results on the Banking77 dataset.

% Table created by stargazer v.5.2.2 by Marek Hlavac, Harvard University. E-mail: hlavac at fas.harvard.edu
% Date and time: Sun, Apr 23, 2023 - 4:04:12 AM 
\begin{small}
\begin{longtable}{c*{7}{>{\centering\arraybackslash}p{1.5cm}}} 
  \caption{Mean [95\% confidence interval] absolute error of accuracy prediction for Banking77 data set.} 
  \label{resbanking77}
\small
\\
\toprule
\# Dimensions & ATC-JS & ATC-L1 & ATC-L2 & ATC-L2n & ATC-Max & ATC-NE & DoC \\ 
\midrule
2 & 11.02 [1.25,22.50] & 11.02 [1.25,22.50] & 11.02 [1.25,22.50] & 11.02 [1.25,22.50] & 11.02 [1.25,22.50] & 11.02 [1.25,22.50] & 36.19 [1.53,94.85] \\ 
3 & 4.25 [0.83,10.00] & 4.26 [0.83,10.00] & 4.16 [0.00,10.00] & 4.16 [0.00,10.00] & 4.26 [0.83,10.00] & 4.15 [0.00,10.00] & 21.00 [0.74,54.33] \\ 
4 & 12.08 [1.84,26.88] & 11.45 [0.62,26.92] & 12.51 [0.62,26.25] & 12.51 [0.62,26.25] & 13.56 [3.12,29.38] & 12.25 [1.25,26.25] & 27.18 [1.07,70.96] \\ 
5 & 9.74 [0.50,22.50] & 10.06 [1.00,20.50] & 10.38 [1.00,24.00] & 10.38 [1.00,24.00] & 16.33 [4.00,26.00] & 9.92 [1.00,22.50] & 25.26 [0.89,67.67] \\ 
6 & 9.95 [2.08,22.93] & 10.90 [2.07,25.00] & 10.15 [2.50,21.67] & 10.15 [2.50,21.67] & 11.11 [0.00,21.25] & 10.00 [2.08,22.92] & 21.16 [0.80,57.49] \\ 
7 & 15.18 [6.79,21.82] & 14.02 [6.07,25.00] & 15.91 [6.79,23.57] & 15.91 [6.79,23.57] & 15.61 [2.14,23.93] & 15.47 [6.79,23.21] & 20.92 [1.29,57.63] \\ 
8 & 7.89 [0.00,18.12] & 7.54 [1.23,19.38] & 8.23 [0.62,19.69] & 8.23 [0.62,19.69] & 8.64 [0.31,18.12] & 8.01 [0.31,18.12] & 18.85 [0.75,51.12] \\ 
9 & 13.53 [3.89,24.17] & 14.13 [4.17,25.56] & 13.21 [3.33,23.89] & 13.21 [3.33,23.89] & 14.20 [3.61,21.94] & 13.39 [3.89,23.89] & 18.07 [1.00,52.34] \\ 
10 & 8.35 [1.25,17.25] & 9.37 [1.00,18.25] & 8.12 [0.50,17.25] & 8.12 [0.50,17.25] & 6.39 [0.75,14.00] & 8.32 [1.00,17.00] & 16.52 [0.70,44.82] \\ 
11 & 6.34 [0.22,14.55] & 5.26 [0.45,11.59] & 7.38 [0.45,13.41] & 7.38 [0.45,13.41] & 12.26 [2.73,17.95] & 6.83 [0.91,14.32] & 16.63 [0.53,43.67] \\ 
12 & 15.07 [6.01,21.25] & 13.07 [4.58,19.17] & 13.93 [5.60,22.71] & 13.94 [5.60,22.72] & 11.03 [2.71,20.42] & 14.84 [6.65,21.46] & 15.60 [0.95,43.44] \\ 
13 & 14.65 [8.26,20.19] & 11.40 [4.04,17.88] & 14.61 [10.35,21.35] & 14.62 [10.35,21.35] & 14.15 [5.96,20.19] & 14.96 [8.65,20.58] & 16.27 [0.79,44.72] \\ 
14 & 14.41 [9.46,19.64] & 16.66 [10.00,25.89] & 14.81 [9.29,20.91] & 14.82 [9.29,21.61] & 14.28 [7.32,22.68] & 14.53 [8.75,19.29] & 15.85 [0.58,42.72] \\ 
15 & 6.98 [0.50,14.83] & 7.68 [0.67,16.17] & 7.08 [1.83,15.50] & 7.09 [2.16,15.50] & 5.18 [0.00,13.17] & 7.02 [1.00,14.50] & 13.47 [0.29,36.89] \\ 
16 & 10.17 [2.03,17.03] & 10.32 [3.28,17.34] & 10.17 [1.88,17.03] & 10.18 [1.88,17.03] & 12.37 [3.44,18.28] & 10.33 [1.25,16.56] & 15.25 [0.85,40.44] \\ 
17 & 12.98 [5.29,18.68] & 12.09 [3.37,18.38] & 12.30 [5.29,19.27] & 12.31 [5.29,19.27] & 10.89 [2.50,18.53] & 12.83 [5.15,18.68] & 13.87 [0.38,39.30] \\ 
18 & 15.93 [10.69,21.11] & 15.60 [9.03,22.22] & 15.60 [9.58,21.39] & 15.61 [9.58,21.25] & 13.99 [4.31,20.84] & 15.81 [9.86,20.97] & 15.75 [0.56,39.51] \\ 
19 & 15.22 [7.76,21.71] & 13.55 [7.76,19.61] & 15.71 [7.49,23.82] & 15.71 [7.62,23.82] & 16.21 [9.47,21.84] & 15.70 [6.84,21.97] & 14.83 [0.55,40.29] \\ 
20 & 11.17 [6.12,17.88] & 12.64 [6.50,18.75] & 10.54 [5.38,17.75] & 10.54 [5.38,17.75] & 11.50 [2.12,17.62] & 10.95 [5.12,17.75] & 12.88 [0.54,34.25] \\ 
21 & 10.87 [2.86,17.50] & 9.28 [1.31,15.12] & 11.59 [5.00,18.10] & 11.59 [5.00,18.10] & 11.73 [6.06,17.98] & 10.99 [2.50,18.33] & 12.11 [0.58,32.72] \\ 
22 & 8.14 [2.84,12.62] & 8.31 [3.86,13.53] & 8.02 [2.05,13.30] & 8.03 [2.05,13.30] & 8.03 [1.36,13.30] & 8.13 [2.95,13.64] & 11.21 [0.73,30.96] \\ 
23 & 14.91 [9.02,20.11] & 16.21 [11.29,20.76] & 13.72 [7.28,21.63] & 13.72 [7.28,21.63] & 11.06 [4.67,16.85] & 14.66 [9.02,20.87] & 13.08 [0.63,35.23] \\ 
24 & 17.30 [10.00,23.44] & 17.79 [11.35,24.69] & 17.57 [11.15,23.12] & 17.58 [11.15,23.12] & 16.40 [8.23,24.06] & 17.32 [10.10,23.33] & 15.14 [0.70,39.56] \\ 
25 & 9.57 [3.80,16.00] & 9.16 [4.60,15.70] & 9.50 [4.60,15.70] & 9.51 [4.60,15.70] & 10.44 [4.50,15.30] & 9.59 [4.10,15.50] & 11.91 [0.52,33.70] \\ 
26 & 14.14 [7.69,19.52] & 14.91 [8.85,19.42] & 14.88 [5.87,21.25] & 14.89 [5.87,21.25] & 13.04 [5.58,20.29] & 14.51 [7.40,20.38] & 12.53 [0.39,32.96] \\ 
27 & 9.45 [5.83,12.97] & 10.18 [6.30,12.96] & 8.69 [4.17,12.78] & 8.70 [4.17,12.78] & 10.90 [3.15,17.13] & 8.18 [4.35,13.24] & 10.64 [0.31,29.46] \\ 
28 & 11.45 [7.41,16.34] & 11.37 [5.62,15.80] & 12.68 [6.70,18.93] & 12.68 [6.70,18.93] & 8.62 [3.21,16.34] & 12.26 [7.77,17.41] & 11.21 [0.48,30.88] \\ 
29 & 13.67 [9.66,18.53] & 12.95 [8.79,16.55] & 14.85 [9.83,20.09] & 14.85 [9.83,20.09] & 14.40 [7.84,18.62] & 13.98 [9.74,18.02] & 12.20 [0.53,31.21] \\ 
30 & 9.85 [5.00,15.33] & 12.44 [7.42,17.25] & 9.50 [4.08,14.92] & 9.51 [4.08,14.92] & 9.59 [4.58,14.50] & 9.65 [4.49,14.00] & 10.85 [0.66,29.41] \\ 
31 & 12.91 [8.15,16.87] & 12.29 [7.90,15.97] & 10.82 [5.97,16.29] & 10.82 [5.97,16.29] & 8.21 [2.50,14.52] & 12.07 [6.05,16.61] & 10.12 [0.39,27.73] \\ 
32 & 11.61 [7.27,15.16] & 12.50 [8.05,16.56] & 10.86 [7.34,14.61] & 10.86 [7.34,14.61] & 10.33 [4.21,14.77] & 11.50 [7.58,13.98] & 10.27 [0.41,26.97] \\ 
33 & 7.98 [3.79,12.20] & 9.20 [4.55,14.55] & 6.50 [3.11,10.61] & 6.50 [3.11,10.61] & 5.47 [0.76,9.93] & 7.15 [3.26,11.36] & 9.05 [0.34,25.48] \\ 
34 & 9.90 [5.81,13.82] & 10.85 [7.20,15.59] & 10.46 [6.47,13.98] & 10.47 [6.47,14.26] & 9.44 [4.93,12.87] & 9.95 [6.03,13.90] & 10.20 [0.45,27.38] \\ 
35 & 8.24 [4.50,11.71] & 9.87 [5.21,14.00] & 6.97 [3.07,10.71] & 6.97 [3.07,10.71] & 7.02 [2.79,11.71] & 7.86 [3.50,10.93] & 9.83 [0.35,27.80] \\ 
36 & 5.81 [1.67,10.56] & 8.89 [4.58,12.36] & 5.97 [2.01,11.74] & 5.97 [2.01,11.74] & 6.72 [2.08,12.36] & 5.60 [1.32,10.76] & 9.09 [0.31,25.31] \\ 
37 & 9.80 [6.55,12.97] & 10.34 [6.89,14.67] & 10.32 [5.95,14.12] & 10.32 [5.95,14.12] & 7.73 [3.11,12.91] & 10.12 [6.76,13.38] & 9.14 [0.39,25.21] \\ 
38 & 7.42 [2.11,12.83] & 10.90 [4.74,15.72] & 6.66 [1.12,12.90] & 6.66 [1.12,12.90] & 8.22 [3.95,12.44] & 6.81 [2.37,12.04] & 9.64 [0.45,25.95] \\ 
39 & 6.41 [1.22,11.67] & 6.41 [1.03,11.35] & 6.31 [1.47,10.64] & 6.31 [1.47,10.64] & 5.33 [0.64,9.29] & 6.50 [1.09,11.35] & 8.77 [0.32,23.10] \\ 
40 & 6.77 [2.00,11.44] & 7.25 [2.50,13.06] & 6.46 [2.75,10.38] & 6.46 [2.75,10.38] & 4.08 [0.25,8.64] & 6.91 [1.87,10.25] & 8.41 [0.40,23.84] \\ 
41 & 9.86 [5.00,14.21] & 13.37 [9.57,16.46] & 9.27 [5.18,13.41] & 9.27 [5.18,13.41] & 6.85 [2.32,11.71] & 9.89 [4.63,14.88] & 8.73 [0.30,23.70] \\ 
42 & 10.36 [7.14,14.82] & 10.53 [6.43,14.46] & 8.87 [5.42,12.14] & 8.87 [5.42,12.14] & 8.12 [4.35,12.32] & 9.67 [6.25,13.51] & 8.68 [0.39,24.11] \\ 
43 & 9.19 [4.82,13.55] & 9.91 [6.74,13.31] & 9.32 [4.83,13.72] & 9.32 [4.83,13.72] & 8.81 [3.95,13.02] & 9.65 [5.04,13.20] & 9.04 [0.41,24.58] \\ 
44 & 14.83 [11.08,17.61] & 13.81 [9.77,18.07] & 12.92 [9.20,16.93] & 12.92 [9.20,16.93] & 11.73 [7.33,16.43] & 14.12 [10.00,16.99] & 11.46 [0.45,27.87] \\ 
45 & 10.83 [7.00,14.40] & 10.62 [5.83,15.06] & 11.41 [6.50,15.61] & 11.41 [6.50,15.61] & 9.57 [4.94,14.11] & 11.49 [7.78,15.89] & 9.21 [0.52,24.59] \\ 
46 & 9.87 [6.90,13.37] & 10.08 [6.30,13.42] & 10.51 [7.17,14.51] & 10.52 [7.17,14.51] & 9.20 [3.85,13.70] & 10.17 [6.68,13.10] & 9.63 [0.40,25.13] \\ 
47 & 9.75 [4.95,13.72] & 9.22 [6.06,12.66] & 8.49 [4.52,13.40] & 8.49 [4.52,13.41] & 7.02 [1.65,11.70] & 9.12 [4.26,14.52] & 8.35 [0.50,22.93] \\ 
48 & 8.86 [4.64,13.18] & 8.45 [5.62,11.72] & 9.47 [4.74,12.81] & 9.47 [4.74,12.81] & 7.51 [3.49,11.25] & 8.91 [5.31,13.54] & 8.63 [0.37,23.22] \\ 
49 & 9.33 [5.51,13.11] & 8.24 [5.00,11.28] & 7.45 [4.69,11.12] & 7.45 [4.69,11.12] & 7.21 [3.82,10.71] & 8.65 [5.15,11.79] & 8.75 [0.33,23.38] \\ 
50 & 10.19 [7.50,13.30] & 10.06 [6.30,14.75] & 9.13 [5.75,13.90] & 9.13 [5.75,13.90] & 8.00 [4.05,13.10] & 10.04 [7.05,12.80] & 8.80 [0.31,23.04] \\ 
51 & 6.90 [2.40,11.18] & 5.98 [2.55,9.46] & 8.17 [4.16,12.26] & 8.17 [4.16,12.26] & 7.00 [3.43,11.67] & 7.54 [3.19,12.60] & 8.24 [0.29,22.34] \\ 
52 & 10.27 [6.54,14.09] & 10.16 [7.07,13.89] & 10.01 [5.53,13.99] & 10.01 [5.53,13.99] & 9.64 [4.52,13.32] & 9.49 [6.44,13.61] & 9.10 [0.42,24.08] \\ 
53 & 11.37 [7.87,14.48] & 12.33 [8.25,15.90] & 11.49 [6.84,14.72] & 11.49 [6.84,14.72] & 9.09 [4.01,13.58] & 11.52 [7.22,14.76] & 9.41 [0.44,25.15] \\ 
54 & 7.87 [4.12,11.95] & 8.99 [5.46,12.08] & 6.96 [2.77,11.02] & 6.96 [2.77,11.02] & 6.61 [2.36,10.83] & 6.70 [2.92,11.71] & 7.79 [0.34,20.88] \\ 
55 & 10.33 [7.32,13.27] & 9.07 [5.86,12.36] & 8.97 [6.04,12.37] & 8.98 [6.04,12.37] & 6.76 [3.82,10.27] & 9.84 [6.73,13.32] & 7.80 [0.32,21.06] \\ 
56 & 7.35 [3.62,11.03] & 5.84 [2.05,9.60] & 7.65 [3.30,11.25] & 7.65 [3.30,11.25] & 6.93 [2.01,11.56] & 7.23 [3.62,10.80] & 7.54 [0.35,20.22] \\ 
57 & 6.28 [2.15,10.44] & 5.99 [2.45,9.47] & 5.52 [1.49,8.42] & 5.52 [1.49,8.42] & 4.66 [0.53,8.73] & 5.74 [1.67,8.86] & 7.15 [0.34,19.90] \\ 
58 & 8.74 [5.82,11.51] & 8.65 [5.43,11.51] & 7.89 [5.30,11.21] & 7.89 [5.30,11.21] & 7.66 [3.36,11.34] & 8.44 [5.52,10.82] & 7.59 [0.39,21.15] \\ 
59 & 6.75 [2.75,10.42] & 5.74 [2.62,9.07] & 6.02 [2.36,10.51] & 6.02 [2.36,10.51] & 5.49 [1.19,10.60] & 6.39 [2.84,10.59] & 7.86 [0.33,21.16] \\ 
60 & 6.29 [3.46,9.21] & 6.98 [2.92,10.71] & 6.42 [3.17,8.96] & 6.42 [3.17,8.96] & 6.49 [1.91,10.54] & 5.94 [3.17,9.54] & 6.80 [0.20,19.85] \\ 
61 & 9.03 [5.57,11.52] & 8.53 [5.45,11.93] & 7.74 [4.75,10.16] & 7.74 [4.75,10.16] & 6.49 [2.83,10.04] & 7.84 [5.41,11.40] & 7.12 [0.29,20.17] \\ 
62 & 6.69 [2.66,10.44] & 6.78 [2.82,11.66] & 7.63 [3.91,10.89] & 7.63 [3.91,10.89] & 5.84 [1.53,10.81] & 7.17 [3.27,11.01] & 7.33 [0.34,20.45] \\ 
63 & 5.12 [1.94,8.77] & 6.57 [3.17,9.25] & 5.13 [2.46,8.30] & 5.13 [2.46,8.30] & 4.09 [0.28,8.06] & 5.33 [1.94,9.80] & 7.18 [0.22,19.68] \\ 
64 & 8.85 [5.03,11.29] & 8.85 [5.16,12.07] & 8.78 [4.92,11.88] & 8.78 [4.92,11.88] & 6.39 [2.50,9.73] & 8.45 [5.04,11.80] & 7.24 [0.25,20.36] \\ 
65 & 7.55 [4.31,11.04] & 8.98 [6.04,11.92] & 6.64 [3.27,10.19] & 6.64 [3.27,10.19] & 5.64 [1.77,10.23] & 7.09 [3.54,10.54] & 7.04 [0.36,20.32] \\ 
66 & 7.36 [4.51,10.04] & 8.54 [5.08,12.46] & 7.06 [4.39,9.85] & 7.06 [4.39,9.85] & 6.58 [3.45,9.28] & 7.37 [4.13,9.62] & 7.10 [0.23,20.07] \\ 
67 & 5.84 [3.21,8.02] & 7.41 [4.70,9.93] & 5.65 [2.50,8.10] & 5.65 [2.50,8.10] & 4.93 [0.82,8.99] & 6.05 [2.76,8.51] & 6.74 [0.21,18.22] \\ 
68 & 6.57 [4.45,9.49] & 7.78 [4.23,11.18] & 5.53 [2.50,9.26] & 5.53 [2.50,9.26] & 4.16 [0.26,7.87] & 6.67 [4.01,9.49] & 6.71 [0.34,17.90] \\ 
69 & 7.35 [4.20,11.56] & 8.53 [4.71,11.67] & 7.54 [4.78,10.62] & 7.54 [4.78,10.62] & 6.47 [2.50,10.47] & 7.61 [4.82,10.88] & 6.89 [0.21,19.35] \\ 
70 & 6.44 [3.50,9.61] & 4.59 [0.93,7.79] & 7.36 [3.46,10.64] & 7.36 [3.46,10.64] & 6.77 [3.00,10.39] & 6.90 [3.61,11.22] & 7.28 [0.30,19.11] \\ 
71 & 7.16 [4.26,9.86] & 4.10 [1.34,6.59] & 7.00 [4.08,10.21] & 7.00 [4.08,10.21] & 6.03 [2.46,9.23] & 6.97 [4.23,10.18] & 6.77 [0.24,18.51] \\ 
72 & 5.29 [2.40,8.65] & 5.18 [2.22,7.68] & 4.37 [1.63,7.22] & 4.37 [1.63,7.22] & 4.38 [0.76,8.12] & 5.28 [2.05,8.58] & 6.42 [0.29,17.73] \\ 
73 & 7.11 [3.59,10.51] & 6.90 [4.04,10.27] & 7.11 [3.84,10.27] & 7.11 [3.84,10.27] & 5.67 [1.16,9.11] & 7.16 [3.83,10.27] & 6.81 [0.26,19.34] \\ 
74 & 5.74 [3.01,9.05] & 6.79 [3.55,8.99] & 5.66 [2.50,8.24] & 5.66 [2.50,8.24] & 7.05 [2.94,10.41] & 5.66 [2.73,8.86] & 6.80 [0.29,19.09] \\ 
75 & 5.25 [2.23,8.20] & 5.62 [2.93,8.17] & 4.63 [1.00,7.94] & 4.63 [1.00,7.94] & 4.77 [1.00,7.23] & 5.66 [1.73,8.43] & 6.23 [0.25,17.98] \\ 
76 & 5.87 [2.47,9.18] & 6.25 [3.45,8.95] & 4.79 [1.87,8.16] & 4.80 [1.87,8.16] & 3.50 [0.99,7.17] & 5.88 [2.70,8.98] & 6.12 [0.22,17.48] \\ 
77 & 4.10 [0.32,7.24] & 4.67 [1.46,7.92] & 1.92 [0.10,5.07] & 1.92 [0.10,5.07] & 1.86 [0.00,4.74] & 4.12 [0.97,6.63] & 6.00 [0.19,17.90] \\
\bottomrule
\end{longtable} 
\end{small}

\end{document}